\documentclass{article}

\PassOptionsToPackage{numbers, compress}{natbib}

\usepackage[preprint]{neurips_2023}

\usepackage[utf8]{inputenc} 
\usepackage[T1]{fontenc}    
\usepackage{hyperref}       
\usepackage{url}            
\usepackage{booktabs}       
\usepackage{amsfonts}       
\usepackage{nicefrac}       
\usepackage{microtype}      
\usepackage{xcolor}         

\usepackage{amsmath}
\usepackage{amssymb}
\usepackage{mathtools}
\usepackage{amsthm}
\usepackage{caption}

\usepackage[capitalize,noabbrev]{cleveref}
\usepackage{multirow}
\usepackage{enumitem}

\setlist[itemize]{label=$\triangleright$,leftmargin=6mm,itemsep=0pt,topsep=0pt}

\theoremstyle{plain}
\newtheorem{theorem}{Theorem}[section]

\newtheorem{lemma}[theorem]{Lemma}

\theoremstyle{definition}
\newtheorem{definition}[theorem]{Definition}

\theoremstyle{remark}

\usepackage[textsize=tiny]{todonotes}
\usepackage{algorithm}
\usepackage{algorithmic}

\newcommand{\poly}{\ensuremath{\mathsf{poly}}}
\newcommand{\E}{\mathbb{E}}

\newcommand{\ALG}{\textsf{ALG}}

\newcommand{\diam}{\mathsf{diam}}

\newcommand{\A}{\mathcal{A}}

\newcommand{\R}{\mathcal{R}}
\newcommand{\X}{\mathcal{X}}
\newcommand{\U}{\mathcal{U}}

\newcommand{\calS}{\mathcal{S}}

\newcommand{\risk}{\mathsf{risk}}
\newcommand{\gen}{\mathsf{gen}}
\newcommand{\opt}{\mathsf{opt}}
\newcommand{\apprx}{\mathsf{approx}}
\newcommand{\rSGD}{\mathsf{rSGD}}
\newcommand{\NSGD}{\mathsf{NSGD}}
\newcommand{\NSSGD}{\mathsf{NSSGD}}

\newcommand{\og}{\overline{g}}
\newcommand{\tg}{\Tilde{g}}

\newcommand{\epsl}{\varepsilon}

\newcommand{\Var}{\mathrm{Var}}

\title{Learning across Data Owners with Joint Differential Privacy}

\author{%
  Yangsibo Huang\\
  Princeton University \\
  \texttt{yangsibo@princeton.edu} \\
  \And
  Haotian Jiang\\
  Microsoft Research\\ \texttt{jhtdavid@cs.washington.edu}\\
  \And
  Daogao Liu\\
  University of Washington\\
  \texttt{dgliu@uw.edu}\\
  \And 
  Mohammad Mahdian\\
  Google Research\\
  \texttt{mahdian@google.com}\\
  \And
  Jieming Mao\\
  Google Research\\
  \texttt{maojm@google.com}\\
  \And 
  Vahab Mirrokni \\
  Google Research\\
  \texttt{mirrokni@google.com}
}

\graphicspath{{./figures/}}

\begin{document}

\maketitle


\begin{abstract}
 In this paper, we study the setting in which data owners train machine learning models collaboratively under a privacy notion called joint differential privacy \cite{kpru14}. In this setting, the model trained for each data owner $j$ uses $j$'s data without privacy consideration and other owners' data with differential privacy guarantees. This setting was initiated in \cite{jrs+21} with a focus on linear regressions. In this paper, we study this setting for stochastic convex optimization (SCO). We present an algorithm that is a variant of DP-SGD \cite{SCS13,ACG+16} and provides theoretical bounds on its population loss. We compare our algorithm to several baselines and discuss for what parameter setups our algorithm is more preferred. We also empirically study joint differential privacy in the multi-class classification problem over two public datasets. Our empirical findings are well-connected to the insights from our theoretical results.
\end{abstract}

\section{Introduction}
In recent years, there are growing attention to ensuring privacy when user data owners share data for jointly training machine learning models. These owners could be single users who have their personalized models trained on their own devices, or online platforms, who are prevented from sharing data with other platforms due to user requirements or data privacy laws and regulations (e.g. the Digital Markets Act (Article 5(2)) \citep{dma} aims to regulate how each Big Tech firm shares data collected from different service platforms). 

In many of these scenarios, the model trained for each data owner is only used by itself. And therefore, it is only necessary to ensure that its data are used under privacy guarantees when training models for other owners, but not the model for itself.

In differential privacy research, this is precisely captured by a notion called "joint differential privacy" \cite{kpru14}, which was initially developed for applications in implementing approximate equilibrium in game theory. Informally, joint differential privacy requires that when one data owner switches to reporting a neighboring dataset, the output distribution of other data owners' outcomes, machine learning models in this paper's context, should not change by much.

Machine learning with joint differential privacy was initiated in \cite{jrs+21}, focusing on linear regressions. Their main result is an alternating minimization algorithm that alternates between users updating individual predictor components non-privately and users jointly updating the common embedding component with differential privacy. In this work, we consider stochastic convex optimization (SCO), the basis of many machine learning tasks, under joint differential privacy (joint-DP). 

\subsection{Our Contributions}

\paragraph{Problem setting.} 
Consider a set of $n$ owners, and each owner $j \in [n]$ contains $r$ users who jointly hold a dataset of $m$ records $\mathcal{S}_j = \{z_{ij}\}_{i \in [m]}$ where $z_{ij} \in \mathcal{Z} \subseteq \mathbb{R}^d$ are drawn i.i.d. from a distribution $\mathcal{P}_j$ over the domain $\mathcal{Z}$. 
For concreteness, let us assume that each user of the same owner holds the same number $m/r$ of data points in the dataset $\mathcal{S}_j$. 
Let $\mathcal{S} = \{\mathcal{S}_j\}_{j \in [n]}$ denote the collection of all owners' data. 

To characterize the personalized model for each owner and shared information across different owners, we assume the set of parameters is divided into two parts: each owner $j$'s personalized parameters $x_j$, and common parameters $u$ that are shared by all owners in the system. 
We let $\X \subset \mathbb{R}^k$ be the domain for each owner's parameters $x_j$, and $\U \subset \mathbb{R}^\ell$ the domain of the parameters $u$ shared by all owners. Denote $D_{\X} = \diam(\X)$ and $D_{\U} = \diam(\U)$ the diameters of these domains.

Given an $L$-Lipschitz loss function $h(x,u,z): \X \times \U \times \mathcal{Z} \rightarrow \mathbb{R}$ that is convex in $(x,u)$ for any $z \in \mathcal{Z}$. 
We define the excess population loss function of each owner $j$ as
\begin{align*}
f_j(x_j, u) &:= \E_{z_j \sim \mathcal{P}_j}[h(x_j, u, z_j)] 
\end{align*}
and the excess population loss function of all owners as 
\begin{align}
\label{eq:objective_function}
    f(x,u) = \E_{j \sim [n]} [f_j(x_j, u)],
\end{align}
where $x=[x_1,x_2,\cdots,x_n]\in\mathbb{R}^{nk}$ is the concatenation of the $n$ personalized parameters. Intuitively, the excess population loss function $f$ is the average excess population loss function of all the owners. 

We aim to learn the personalized and common parameters $\{x_j^*\}_{j\in[n]}, u^*$ to minimize the excess population  function $f$ while satisfying joint differential privacy across different owners.
Intuitively, the parameter $x_j^* \in \X$ corresponds to each owner $j$'s best personalization, and the parameter $u^* \in \U$ corresponds to the commonalities among different owners. 
To capture the shared information among the users, we assume that the minimizer $(x_j^*,u^*)$ to each owner $j$'s population loss $f_j(x_j, u)$ shares the same $u^*$ for any $j \in [n]$. 

In practical applications, it is often  instructive to consider the case where $\ell \gg k$, i.e., the amount of shared information is much larger than the amount of personalized information, but we do not make this assumption and the statement of our results hold for all regimes of $\ell$ and $k$.

\paragraph{User-Level Joint Differential Privacy.} To satisfy the joint-DP guarantees, we require that each owner $j$'s data is protected from any other owner $j'$ in the learning of the personalized and shared parameters, while owner $j$ can learn and compute non-privately using her own data. 
In particular, this means that different users of the {\em same} owner can learn and compute non-privately using each others' data, but user's data are protected from any user belonging to a {\em different} owner. 

As in \cite{jrs+21}, we operate in the billboard model where apart from the $n$ owners introduced above, there is additionally a computing server. The server runs a differentially private algorithm on sensitive information from the owners, and broadcasts the output to all the owners. Each owner $j \in [n]$ can then use the broadcasted output in a computation that solely relies on her data. The output of this computation is not made available to other owner. The billboard model is that it trivially satisfies joint differential privacy.

We consider the standard setting of user-level joint-DP. This means that the replacement or not of a single user's data in an owner's dataset should be indistinguishable to every other owner in the system (see \Cref{defn:DP} for a more formal definition of indistinguishability and differential privacy). 
%



\paragraph{Our Theoretical Result.} Our main result is an algorithm for user-level joint-DP with optimal excess population loss. In particular, we prove the following theorem. 


\begin{theorem}[User-level joint-DP, Informal] \label{thm:main_informal}
For any $0 < \delta < 1/\poly(n)$ and $10\ge \epsl > 0$, there is an algorithm for the problem above achieving  user-level $(\epsl,\delta)$ joint-DP  with excess population loss
\[
O\Big(RL \cdot \Big(\frac{1}{\sqrt{mn}} + \frac{\sqrt{\ell \log(m/\delta)}}{\epsl n r} \Big) \Big) ,
\]
where $R=\sqrt{n D_{\X}^2 + D_{\U}^2}$.
\end{theorem}

The formal statement of the theorem will be given in \Cref{thm:main}.

\paragraph{Empirical Studies.} In Section \ref{sec:experiment}, we conduct empirical studies on the multi-class calssification problem over two public datasets: MNIST~\cite{l98} and FashionMNIST~\cite{x17}. We compare joint-DP models to full-DP model and per-silo models (users learning individually), in terms of test accuracy. The comparison results are consistent with the discussion we get from comparing our theoretical bounds in Section \ref{subsec:discussion}.

\subsection{Related Work}
The most related paper is \cite{jrs+21}, which studies a very similar joint differential privacy setup as our paper on linear regressions. There's a minor difference in the setting in which our paper distinguishes between the notion of data owners and users. And their paper's setting is equivalent to the case when each data owner is a single user in our paper.

Stochastic convex optimization under (normal) differential privacy has been studied extensively \cite{BassilyST14, BassilyFTT19,FKT20}, and there is a long line of work on this.
It is still a recent hotspot for the DP optimization community.
Some of the representative research topics in this direction include studying the non-Euclidean geometries \cite{afkt21,bgn21,bgm21,GLL+22}, improving the gradient computation complexity to achieve optimal privacy-utility tradeoff \cite{kll21,CJJ+23}, 
getting dimension-independent bounds with additional assumptions \cite{sstt21,LLH+22}, dealing with heavy-tailed data \cite{WXDX20,JZH+22,KLZ22} and so on.

The privacy notion ``joint differential privacy'' was initiated in \cite{kpru14} and it is followed up in \cite{RR14,hHRRW14,KMRR15,CKRW15}. Their main focus is to implement various types of equilibria. 

Recent studies on user-level (full) differential privacy \cite{LSA+21,EMN21,GKK+23} have introduced intriguing techniques that could possibly be applied in the joint-DP setting we explore in this paper. In these studies, each owner holds only one user's data ($r=1$), and only shared parameters $u$ are considered. Specifically, in scenarios where the functions are smooth, the results in \cite{LSA+21} can be adapted to achieve an excess population loss of $\Tilde{O}(RL(\frac{1}{\sqrt{mn}}+\frac{\ell}{\epsilon n\sqrt{mr}}))$. While this bound offers improved dependence on $m/r$,
it exhibits worse dependence on the dimension $\ell$. Further details can be found in the Appendix.

\section{Preliminaries}

In this section, we give formal definitions of differential privacy (full-DP) and joint differential privacy (joint-DP). We start with full-DP.

\begin{definition}[Differential Privacy and indistinguishable]
\label{defn:DP}
We say two random variables $X,Y$ is $(\epsl,\delta)$-indistinguishable if for any event $\mathcal{E}$, we have $\Pr[X\in \mathcal{E}]\le e^{\epsl}\Pr[Y\in\mathcal{E}]+\delta$ and $\Pr[Y\in\mathcal{E}]\le e^{\epsl}\Pr[X\in\mathcal{E}]+\delta$.
We say an algorithm $\A$ is $(\epsl,\delta)$-differentially private (DP) if for all neighboring datasets $\calS$ and $\calS'$, $\A(\calS)$ and $\A(\calS')$ is $(\epsl,\delta)$-indistinguishable.
\end{definition}

The joint-DP definition only poses privacy constraints on other owners' models $A_{-j}$, when the neighboring datasets differ in data from owner $j$. In our case, $A_{-j}$ is just model parameters excluding owner $j$'s parameters $x_j$, i.e., $(u, x_1,...,x_{j-1}, x_{j+1},...,x_n)$.

\begin{definition}[Joint differential privacy]
\label{defn:JDP}
We say an algorithm $\A$ is $(\epsl,\delta)$-jointly differentially private if for all neighboring datasets $\calS$ and $\calS'$ which differ in data owned by owner $j$, $\A_{-j}(\calS)$ and $\A_{-j}(\calS')$ is $(\epsl,\delta)$-indistinguishable.
\end{definition}

 \paragraph{Neighboring datasets.} As we consider user-level DP, we say $\calS$ and $\calS'$ are neighboring datasets if they are the same except for records from one single user. We use record-level DP to indicate the special case when each user owns one data point, i.e. $r = m$.


\section{$(\epsl, \delta)$ Joint-DP Algorithms}
\label{sec:theory}


\subsection{Joint-DP by Uniform Stability}

In this section, we prove the result in \Cref{thm:main_informal} for user-level joint-DP.
Let us start by describing the algorithm we use for our result. 

\noindent \textbf{Sampling-with-Replacement SGD for Record-Level Joint-DP.} 
Let $R = \sqrt{n D_{\X}^2 + D_{\U}^2}$ be the diameter of the parameter space. 
Given dataset $\mathcal{S} = \{\mathcal{S}_1, \cdots, \mathcal{S}_n\}$ where owner $j$'s dataset $\mathcal{S}_j = \{z_{ij}\}_{i \in [m]}$, number of iterations $T$ and stepsizes $\eta_t: t \in [T]$. The algorithm $\A_{\rSGD}$ in each step samples a data point $i_t \in [m]$ and owner $j_t \in [n]$ uniformly at random, and then updates 
\[
(x_{j_t}^{(t+1)}, u^{(t+1)}) = \Pi_{\X}((x_{j_t}^{(t)}, u^{(t)}) - \eta_t \cdot \nabla h(x_{j_t}^{(t)}, u^{(t)}, z_{i_t j_t})). 
\]
Finally, the algorithm returns the parameters $(\overline{x}^{(T)}, \overline{u}^{(T)})  = \frac{1}{\sum_{t=1}^T \eta_t} \sum_{t=1}^T \eta_t (x^{(t)}, u^{(t)})$. In particular, if $\eta_t$ is the same for all $t \in [T]$ which will be the case for us, this is just the average of the parameters in each iteration of the algorithm. 
\begin{algorithm}[h]
\caption{$\A_{\NSGD}$: Noisy SGD for Joint-DP}
\label{alg:NSGD_SCO_jointDP}
\begin{algorithmic}[1]
\STATE Dataset $\mathcal{S}$, $L$-Lipschitz loss $h(\cdot,\cdot)$, parameter $(\epsl, \delta)$, convex sets $(\X, \U)$, learning rate $\eta$  
\STATE Noise variance $\sigma^2 \leftarrow \frac{L^2 T \log(1/\delta)}{\epsl^2 m^2 n^2}$
\STATE Initial parameter $x_0 \in \mathcal{X}$ chosen arbitrarily
\FOR{$t = 0, \cdots, T-1$}
\STATE Sample $i_t \sim [m]$ and $j_t \sim [n]$ uniformly at random 
\STATE $(x_{j_t}^{t+1}, u^{t+1}) \leftarrow \Pi_{\mathcal{X}\times\U}\Big((x_{j_t}^{t}, u^t) - \eta \cdot ( \nabla h(x_{j_t}^t, u^t,$
 $z_{i_t, j_t}) + [\mathbf{0},b_t] )\Big)$, where $\mathbf{0}\in \mathbb{R}^{nk},b_t \sim \mathcal{N}(0, \sigma^2 I_{\U})$.
\ENDFOR
\STATE \textbf{Return} $ (\overline{x}, \overline{u}) := \frac{1}{T} \sum_{t=0}^{T-1} (x^t, u^t)$
\end{algorithmic} 
\end{algorithm}

Our main theoretical result, the guarantee of $\A_{\NSGD}$, is stated below. Here we also state the bound for record-level joint-DP, where each user is assumed to contain a single record, i.e. $m = r$. Our bound for user-level joint-DP, i.e. the full range of $m$ and $r$, is in fact obtained from record-level joint-DP. 
\begin{theorem}[Record-Level and User-Level Joint-DP]
\label{thm:main}
For any $0<\delta < 1/\poly(n)$ and $10\ge \epsl > 0$, $\A(\NSGD)$ is $(\epsl,\delta)$-DP for record-level and user-level joint-DP with population loss
$O\left(RL \cdot (\frac{1}{\sqrt{mn}} + \frac{\sqrt{\ell \log (1/\delta)}}{\epsl m n})\right)$ and $O\left(RL \cdot (\frac{1}{\sqrt{mn}} + \frac{\sqrt{\ell \log(m/r\delta)}}{\epsl n r})\right)$ respectively, where $R=\sqrt{n D_{\X}^2 + D_{\U}^2}$.
\end{theorem}

Recall that $x=[x_1,x_2,\cdots,x_n]\in\mathbb{R}^{nk}$ is the concatenation of the $n$ personalized parameters. 
To finish the proof, we define an empirical function
\begin{align*}
    f_{\calS}(x,u)&:=\frac{1}{nm}\sum_{j\in[n],i\in[m]}h(x_j,u,z_{ij}),
\end{align*}
and let
\begin{align*}
    (x^*(\calS),u^*(\calS)):=\arg\min_{y\in\X,u\in\U}\frac{1}{nm}\sum_{j\in[n],i\in[m]}h(y_j,u,z_{ij}).
\end{align*}

Algorithm $\A_{\NSGD}$ is a variant of DP-SGD \cite{SCS13,ACG+16}. DP-SGD has been studied and used extensively, and the privacy analysis on DP-SGD motivates a long line of work on accounting for the privacy loss of subsampling and compositions \cite{BBG18,WBK19}.
The following lemma can lead to the DP guarantee of $\A_{\NSGD}$ based on the truncated concentrated differential privacy (tCDP) proposed in \cite{BDRS18}.

\begin{lemma}[Apply \cite{kll21}]
\label{lm:kll21}
For $mn\ge10,\epsl<c_1T/(mn)^2$ and $\delta\in(0,1/2]$, $\A_{\NSGD}$ is $(\epsl,\delta)$-DP whenever $\sigma\ge \frac{c_2\sqrt{T\log(1/\delta)}}{\epsl n}$, where $c_1,c_2$ are some universal positive constants.
\end{lemma}

For an algorithm $\A$ and a dataset $\mathcal{S}$, let $\A(\calS)\in \X^m\times\U$ be the (randomized) output of $\A$ with $\calS$ as input.
One can decompose the excess population risk $\phi_{\risk}(\A):=f(\A(\mathcal{S}))-f(x^*,u^*)$ as
\begin{align*}
    \phi_{\risk}(\A) \leq &\underbrace{f(\A(\mathcal{S})) -f_{\mathcal{S}}(\A(\mathcal{S}))}_{\phi_{\gen}(
A)} + 
\underbrace{f_{\mathcal{S}}(\A(\mathcal{S})) - f_{\mathcal{S}}(x^*(\mathcal{S}),u^*(\mathcal{S}))}_{\phi_{\opt}(\A)} + 
\underbrace{f_{\mathcal{S}}(x^*(\mathcal{S}),u^*(\calS)) - f(x^*,u^*)}_{\phi_{\apprx}(\A)} ,
\end{align*}
where we realized that $\E[\phi_{\apprx}(\mathcal{A})] \leq 0$. 
Hence it suffices to consider $\phi_{\gen}$ and $\phi_{\opt}$.
We bound $\phi_{\opt}$ by adopting the analysis of SGD, and bound $\phi_{\gen}$ by considering the uniform stability \cite{bfgt20}.
The relation between generalization error and stability is well-known \cite{BE02}.
We modified the following lemma for bounding $\phi_{\opt}$, whose proof essentially follows from the seminar work \cite{BE02}. 

\begin{lemma}[Uniform Stability] 
\label{lem:uniform_stability}
Given a learning algorithm $\A$, a dataset $\mathcal{S} = \{S_1,\cdots, S_n\}$ formed
by $m$ i.i.d. samples $S_j$ drawn from the $j$th user's distribution $\mathcal{P}_j$.
We replace one random data point $z \in S_j$ for a uniformly random $j \sim [n]$ by another independent data point $z' \sim \mathcal{P}_j$ to obtain $S_j'$ and $\calS'$. Then we have 
\begin{align*}
    \E_{\mathcal{S},\A}[f(\A(\mathcal{S})) -  f_{\mathcal{S}}(\A(\mathcal{S}))] = \E_{\mathcal{S}, S_j', \A}[f_j(\mathcal{A}(\mathcal{S}), z') - f_j(\mathcal{A}(\mathcal{S'}), z')],
\end{align*}
where $\A(\mathcal{S})$ is the output of $\A$ on dataset $\mathcal{S}$.  
\end{lemma}

\begin{proof}
One can verify the statement by checking that the corresponding terms on both sides are the same.

We know $\E_{\calS,\A}[f(\A(\calS))]=\E_{\calS,S_j',\A}[f_j(\A(\calS),z')]$ as $z'$ is independent of $\calS$.
As for the empirical term, 
\begin{align*}
  \E_{\calS,\A}[f_{\calS}(\A(\calS))]=\frac{1}{m}\sum_{j\in[m]}\E_{\calS}[\frac{1}{n}\sum_{z_i\in S_j}h(\A(\calS),z_i)].
\end{align*}
One can rename $z_i$ as $z'$ for any $i$, by the i.i.d. and symmetry assumptions, we get
\begin{align*}
    \E_{\calS,\A}[f_{\calS}(\A(\calS))]=\E_{\mathcal{S}, S_j', \A}[f_j(\mathcal{A}(\mathcal{S'}), z')].
\end{align*}
\end{proof}

The following lemma can be used to bound the uniform stability of our algorithm.
\begin{lemma}[Uniform Stability Key Lemma, \cite{bfgt20}]
Let $(x^t)_{t \in [T]}$ and $(y^t)_{t \in [T]}$ with $x^1 = y^1$ be Gradient Descent trajectories for convex $L$-Lipschitz objectives $(f_t)_{t \in [T-1]}$ and $(f_t')_{t \in [T-1]}$, respectively; i.e.
\begin{align*}
    x^{t+1} &= \Pi_{\mathcal{X}}(x^t - \eta_t \nabla f_t(x^t)) ,\\
    y^{t+1} &= \Pi_{\mathcal{X}}(y^t - \eta_t \nabla f'_t(y^t)) ,
\end{align*}
for all $t \in [T-1]$. Suppose for every $t \in [T-1]$, $\|\nabla f_t(x^t) - \nabla f'_t(x^t)\| \leq a_t$, for scalars $0 \leq a_t \leq 2L$. Then if $t_0 = \inf\{t: f_t \neq f'_t\}$, we have
\[
\E\|x^T - y^T\|_2 \leq 2L \sqrt{\sum_{t= t_0}^{T-1} \eta_t^2} + 2 \sum_{t= t_0+1}^{T-1} \eta_t a_t .
\]
\end{lemma}

Using this lemma, we bound the uniform stability of SGD for joint-DP. 

\smallskip
\begin{lemma}[Uniform Stability Bound for $\A_{\rSGD}$]
\label{lem:stability_rSGD_jointDP}
Let $D_{\X} := \diam(\X)$ and $D_{\U} := \diam(\U)$. 
Let $\delta_{\A_{\rSGD}}(\calS,\calS'):=\|\A_{\rSGD}(\calS)-\A_{\rSGD}(\calS')\|$.
The algorithm $\A_{\rSGD}$ satisfies uniform argument stability bound
\begin{align*}
    &~\sup_{\mathcal{S} \simeq \mathcal{S}'} \E_{\A_{\rSGD}}[\delta_{\A_{\rSGD}}(\mathcal{S}, \mathcal{S}')] \leq \min(R, 4L\eta (\sqrt{T} + \frac{T}{mn}) ).
\end{align*}
where $\mathcal{S}$ and $\mathcal{S}'$ differ by a single record. 
\end{lemma}

The proof is the same as in the vanilla case of DP-SCO. 

\smallskip
\subsubsection{Record-Level Joint-DP}
Now we prove the record-level joint-DP part of Theorem~\ref{thm:main}.

The $(\epsl,\delta)$-DP guarantee for $\A_{\NSGD}$ follows from Lemma~\ref{lm:kll21}.
Now we bound the population loss.

Recall the notation $\og_t$ we defined in Subsection~\ref{subsec:without_DP}, which satisfies that $\|\og_t\|_2^2\le L^2$ and $\E[\og_t]=\nabla f(x^{(t)},u^{(t)})$.
Recall the noise $[\mathbf{0},b_t]$ added in $t$-th iteration.
For simplicity, we let $y^{(t)}=[x^{(t)},u^{(t)}]$, $y^*=[x^*,u^*]$ and $\tg_t=\og_t+[\mathbf{0},b_t]$.
By the convexity, we know
\begin{align*}
    f(y^{(t)})-f(y^*)
    \le & \langle \nabla f(y^{(t)}),y^{(t)}-y^*\rangle =  \langle \tg_t,y^{(t)}-y^*\rangle + \langle \nabla f(y^{[t]})-\tg_t,y^{(t)}-y^*\rangle.
\end{align*}
Note that $y^{(t+1)}=\Pi_{\X^m\times\U}\langle y^{(t)}-\eta\tg_t\rangle$.
Hence we know
\begin{align*}
    \langle \tg_t,y^{(t)}-y^*\rangle\le& \frac{1}{2\eta}(\|y^{(t)}-y^*\|_2^2-\|y^{(t+1)}-y^*\|_2^2) + \frac{\eta}{2}\|\tg_t\|^2_2.
\end{align*}
Taking expectations on both sides, one gets
\begin{align*}
    \E[f(y^{(t)})-y^*]\le& \frac{1}{2\eta}(\|y^{(t)}-y^*\|_2^2-\|y^{(t+1)}-y^*\|_2^2) +\frac{\eta}{2}\E[\|\tg_t\|_2^2].
\end{align*}
Summing the results over $t$, we get
\begin{align*}
\E[\phi_{\opt}(\A_{\NSGD})] &:= \E[f_\mathcal{S}(\overline{x}, \overline{u})] - f_{\mathcal{S}}(x^*, u^*) \leq \frac{R^2}{2\eta T} + \frac{\eta}{2T} \sum_{i=1}^T \E[ \|\tg_t \|_2^2 ] \\
& \leq \frac{R^2}{2\eta T} + \frac{\eta}{2} (L^2 + \ell \sigma^2) =  \frac{R^2}{2\eta T} + \frac{\eta L^2}{2} (1 + \frac{\ell T \log(1/\delta)}{\epsl^2 m^2 n^2}), 
\end{align*} 
where we recall that noise is only added to $\U \subseteq \mathbb{R}^\ell$. 
By the precondition that $\epsl\le O(1)$ and the parameter setting $T=\Theta(m^2n^2)$, we know $ \frac{\ell T \log(1/\delta)}{\epsl^2 m^2 n^2}=\Omega(1)$. 

Next, by Lemma~\ref{lem:uniform_stability}, \Cref{lem:stability_rSGD_jointDP} and the Lipschitz assumptions, $f_j(\A(\calS),z')-f_j(\A(\calS'),z')\le L\sup_{\mathcal{S} \simeq \mathcal{S}'} \E_{\A_{\rSGD}}[\delta_{\A_{\rSGD}}(\mathcal{S}, \mathcal{S}')]$, and hence we have 
\[
\E[\phi_{\gen}] \leq L \cdot \min(R, 4L (\sqrt{T} + \frac{T}{m n})  \eta) \leq \frac{8 L^2 T \eta}{mn} ,
\]
since we have chosen $T \geq m^2n^2$. 
Finally we have $\E[\phi_{\apprx}] \leq 0$. These bounds together imply that 
\begin{align*}
\E[\phi_{\risk}(\A_{\NSGD})] 
 \lesssim \frac{R^2}{\eta T} + \eta L^2 T \cdot \Big(\frac{1}{mn} + \frac{\ell \log(1/\delta)}{\epsl^2 m^2 n^2} \Big)  
\leq RL \Big(\frac{1}{\sqrt{mn}} + \frac{\sqrt{\ell \log(1/\delta)}}{\epsl mn} \Big) .
\end{align*}

\subsubsection{User-Level Joint-DP}

We can immediately use the above record-level joint-DP bound to obtain an
optimal user-level joint-DP bound by using a notion known as Group Privacy. 

\begin{lemma}[Group Privacy,\cite{dr14}] \label{lem:group_privacy}
If some algorithm $\A$ is $(\epsl,\delta)$-DP for any neighboring datasets $\calS,\calS'$ differing one sample, then $\A(\calS)$ and $\A(\calS')$ are $(k\epsl,ke^{k\epsl}\delta)$-indistinguishable when $\calS,\calS'$ differ by $k$ samples.
\end{lemma}

In particular, since each user holds $m/r$ records of its owner's dataset, if we want to satisfy user-level $(\epsl, \delta)$-joint-DP, then we simply apply the record-level $(\epsl', \delta')$-joint-DP algorithm above with $\epsl' = \epsl r/m$ and $\delta' = O(\delta r/m)$ would give the correct bound of $RL \cdot (\frac{1}{\sqrt{mn}} + \frac{\sqrt{\ell \log(m/ r\delta)}}{\epsl n r})$ by directly applying Group Privacy in \Cref{lem:group_privacy}. 

Similarly, if we consider a change of a subset of each user's record, we can also get the optimal DP guarantee there. We do not state this more general result explicitly in this paper.

\subsection{Discussions}
\label{subsec:discussion}
We compare our results under joint-DP to other natural settings in the following table, in terms of excess population loss bounds at the record level. Details for deriving these bounds are delayed to the Appendix.
\begin{center}
\begin{tabular}{ |c| c| c| }
\hline

Owners learn individually & Collaboration without DP & Collaboration with full-DP \\ 
\hline
$O(L (D_\X + D_\U)/\sqrt{m})$ & $O\Big( L \cdot \sqrt{\frac{(n D_{\X}^2 + D_{\U}^2)}{mn}}\Big)$ & $O\Big(RL\Big(\frac{1}{\sqrt{mn}}+\frac{\sqrt{(nk+\ell)\log(1/\delta)}}{\epsl mn} \Big)\Big)$ \\  
 \hline
\end{tabular}
\end{center}

The bound under joint-DP always saves a term $O(RL\frac{\sqrt{nk}}{\epsl nm})$ compared to the bound under full-DP, which suggests the advantage of the joint-DP model.
Recall that $R = \sqrt{n D_{\X}^2 + D_{\U}^2}$.

Note that ``owners learn individually''  does not guarantee full-DP but joint-DP with $\epsl= 0$.
If we compare the bound from ``owners learn individually'': $L (D_\X + D_\U)/\sqrt{m}$, and the bound from joint-DP: $L\sqrt{nD_{\X}^2+D_{\U}^2}(\frac{1}{\sqrt{mn}}+\frac{\sqrt{\ell\log(1/\delta)}}{\epsl mn})$,
we can see that there are regimes that the second bound is smaller, which suggests owners to collaborate.
In particular, when $D_{\U}\gg D_{\X}$ and $\frac{\sqrt{\ell\log(1/\delta)}}{\epsl mn}$ is not much bigger than $\frac{1}{\sqrt{mn}}$, joint-DP seems to be preferred, and the larger $n$, the better the bound is.
This is consistent with our experiment observations in the next section.

\section{Experiment Evaluations}
\label{sec:experiment}

In this section, we empirically study the role of collaboration with \textsc{user-level} joint-DP, as well as how it compares to collaboration with full-DP and per-silo training. Although the real-world problems we study in the experiments may not be convex, we find that the empirical results still align with the insights from our theoretical analysis in many ways.

\subsection{Experimental Setup}

\paragraph{Datasets.}

We consider the multi-class classification problem over two public datasets: MNIST~\cite{l98} and
FashionMNIST~\cite{x17}. Each dataset has 10 classes. We use a $N=10,000$ subset of the original training dataset and evenly distribute the training data among $n$ owners.
To simulate real-world scenarios where data distribution among owners is non-IID, we ensure that each owner only owns examples from a maximum of 8 out of the 10 classes. The test data is also divided among the owners in the same manner as the training data, to reflect the non-IID distribution. Note that for this scenario in which the owner's distributions are significantly different, our goal is not to achieve state-of-the-art accuracy. Instead, we aim to compare various training paradigms for obtaining personalization under data heterogeneity and privacy requirements.

\paragraph{Models.} The model we use is a two-layer Convolutional Neural Network (CNN), followed by two parallel fully connected (FC) layers. The final output is obtained by averaging the output from these two FC layers. Unless otherwise specified, in our collaboration with joint-DP experiments, only one FC layer is individually trained by each owner, while the other parameters are shared among owners. We also examined various options for individually trained parameters in Section~\ref{sec:exp_private_choice}.

\paragraph{Baselines.} We compare the collaboration with joint-DP with three baselines:
\begin{itemize}
\item \textbf{Per-silo training} is a baseline where each owner trains a model using only their own data. This approach allows each owner to have a personalized model that is trained specifically on their own data and potentially better suited to their needs. The training process does not require differential privacy as each owner trains a model independently from others, and the private data never leaves the owner's site\footnote{We also assume that the owner would not share the trained models with others.}.
\item \textbf{Collaboration without DP} is a baseline where \textit{all} parameters of the model are collaboratively trained among all owners (in the federated learning fashion), but \textit{without} DP noise added. Our experiments use FedAvg~\cite{medsa17}, and adjust $E$ (the local iterations) and $T$ (the global epochs) for different $n$, according to Table~\ref{tab:fl_config} in Appendix~\ref{sec:app_exp}.
\item \textbf{Collaboration with full-DP} is a baseline where \textit{all} parameters of the model are collaboratively trained among all owners, \textit{with} central-DP noise added. The federated training uses the same setup as Collaboration without DP. The only difference is Collaboration with full-DP applies DP-SGD~\cite{ACG+16} for each global model update.
\end{itemize}

\begin{table}[t]
\setlength{\tabcolsep}{4.5pt}
\begin{center}
\caption{Test accuracy of MNIST and FashionMNIST across 5 different runs for per-silo training ($\epsl=\infty$), collaboration without DP, with full-DP, and with joint-DP. We vary  $n$, the number of owners. The best result among per-silo training, collaboration with full-DP, and collaboration with joint-DP is highlighted in \textbf{boldface}.  When $n > 128$, collaboration with joint-DP achieves better accuracy than per-silo training and collaboration with full-DP.}
\label{tab:main_table}
\begin{small}
\begin{sc}
\begin{tabular}{l|c|c|cccccc}
\toprule
\multirow{2}{*}{$n$} & \multirow{2}{*}{ $m$} &  \textbf{Per-silo} & \multicolumn{3}{c}{\textbf{Collaboration}} \\
&  &  \textbf{Training}& without DP & w/ full-DP ($\epsl=1.0$)  & w/ joint-DP ($\epsl=1.0$) \\ 
\midrule
\multicolumn{6}{c}{Dataset: MNIST} \\
\midrule
$4$               &  $2500$      & $\mathbf{0.9445} \pm 0.0069$     & ${{0.9602}} \pm 0.0017$  & $0.4444 \pm 0.0844$   & ${0.4658} \pm 0.0307$                     \\
$16$              &  $625$       & $\mathbf{ 0.9457} \pm 0.0038$    & ${0.9468} \pm 0.0014$   & $0.2926 \pm 0.0239$   &${0.3207} \pm 0.0690$                      \\
$64$              &  $156.25$    & $\mathbf{ 0.9209} \pm 0.0010$    & ${{0.9402}} \pm 0.0037$  & $0.6982 \pm 0.0314$    &${0.7651} \pm 0.0401$                     \\
$128$             &  $78.13$     & $\mathbf{0.8639} \pm 0.0010$     & ${{0.9376}} \pm 0.0043$  & $0.8081 \pm 0.0240$   &${0.8169} \pm 0.0160$                    \\
$256$             &  $39.06$     & $0.7678 \pm 0.0015$              & ${{0.9196}} \pm 0.0028$  & $0.7399 \pm 0.0324$   &$\mathbf{{0.7835}} \pm 0.0136$                  \\
$512$             &  $19.53$     & $0.7240 \pm 0.0001$              & ${{0.9326}} \pm 0.0093$ &   $0.6023 \pm 0.0621$  &$\mathbf{{0.7423}} \pm  0.0832$          \\
\midrule
\multicolumn{6}{c}{Dataset: FashionMNIST} \\
\midrule
$4$      & $2500$   & $\mathbf{0.8888} \pm 0.0253$       & $0.8759 \pm 0.0042$   &  $0.4560 \pm 0.0142$    & $0.4533 \pm 0.0121$ \\
$16$     & $625$    & $\mathbf{0.8575} \pm 0.0135$       & $0.8301 \pm 0.0078$   &  $0.4044 \pm 0.0473$    & $0.4205 \pm 0.0539$    \\
$64$     & $156.25$ &  $\mathbf{0.8498} \pm 0.0041$      & $0.8232 \pm 0.0088$   &  $0.5354 \pm 0.0249$    & $0.5564 \pm 0.0207$    \\
$128$    & $78.13$  &   $\mathbf{0.7506} \pm 0.0017$     & $0.7676 \pm 0.0123$   &  $0.6593 \pm 0.0137$    & $0.6719 \pm 0.0272$    \\
$256$    & $39.06$  &   $0.6489 \pm 0.0004$     & $0.8010 \pm 0.0018$   &  $0.6749 \pm 0.0160$    & $\mathbf{0.6908} \pm 0.0145$    \\
$512$    & $19.53$  &   $0.6355 \pm 0.0007$     & $0.8242 \pm 0.0069$   & $0.6484 \pm 0.0227$  & $\mathbf{0.6667}  \pm 0.0062$ \\
\bottomrule
\end{tabular}
\end{sc}
\end{small}
\end{center}
\end{table}

\begin{figure*}[t]
    \centering
    \includegraphics[width=0.32\linewidth]{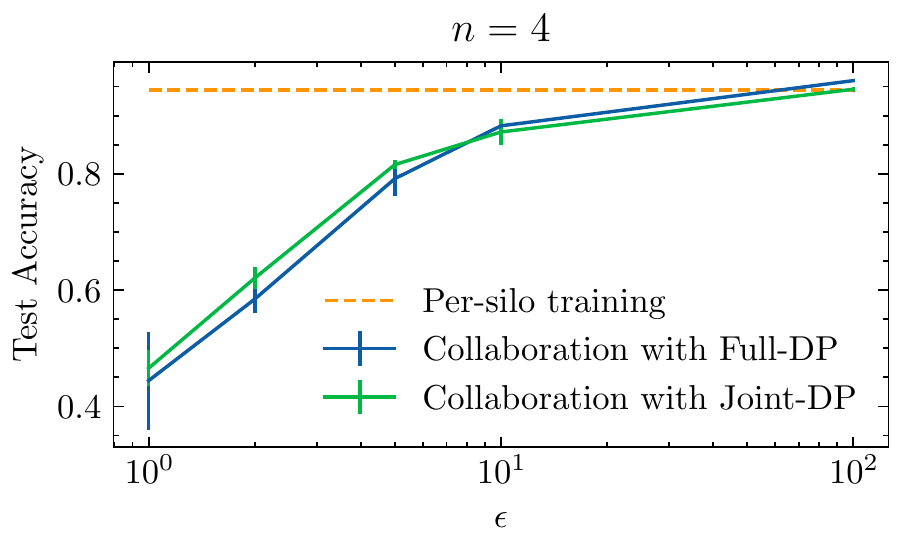}
    \includegraphics[width=0.32\linewidth]{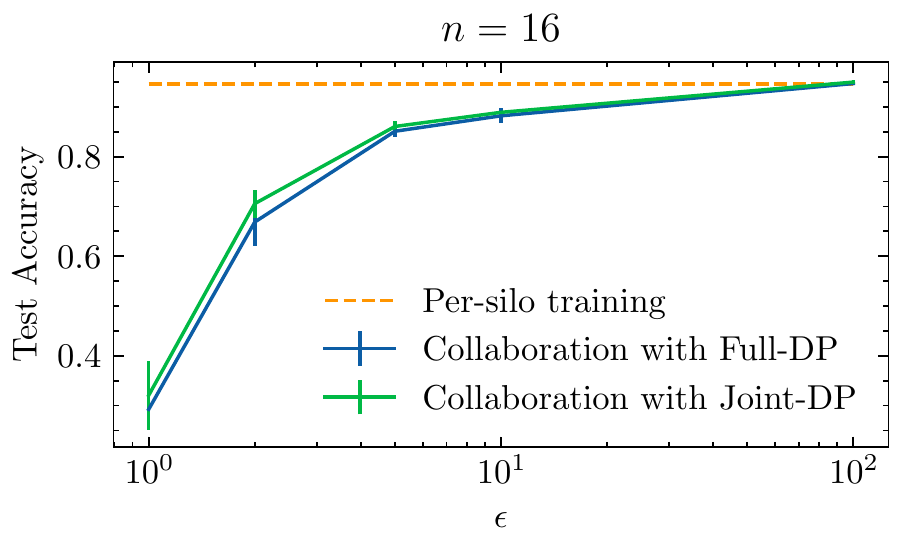}
    \includegraphics[width=0.32\linewidth]{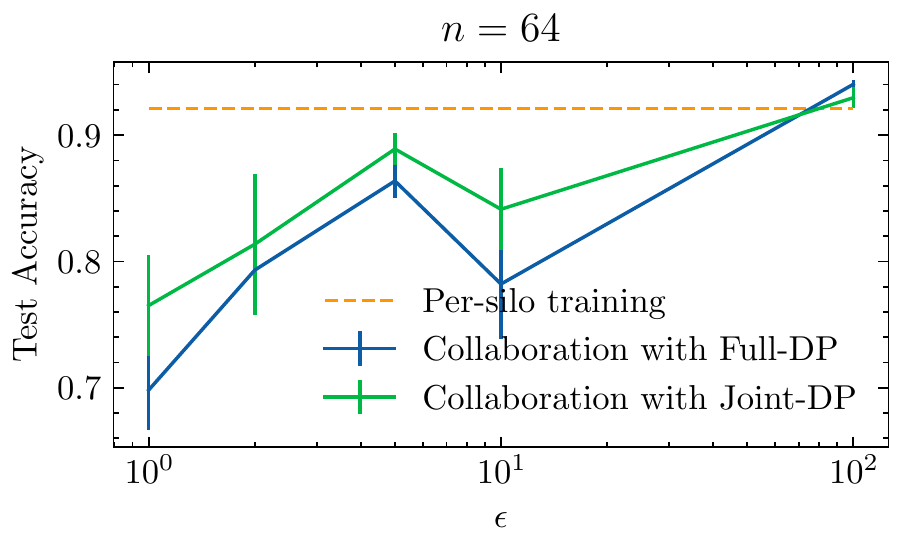}
    
    \includegraphics[width=0.32\linewidth]{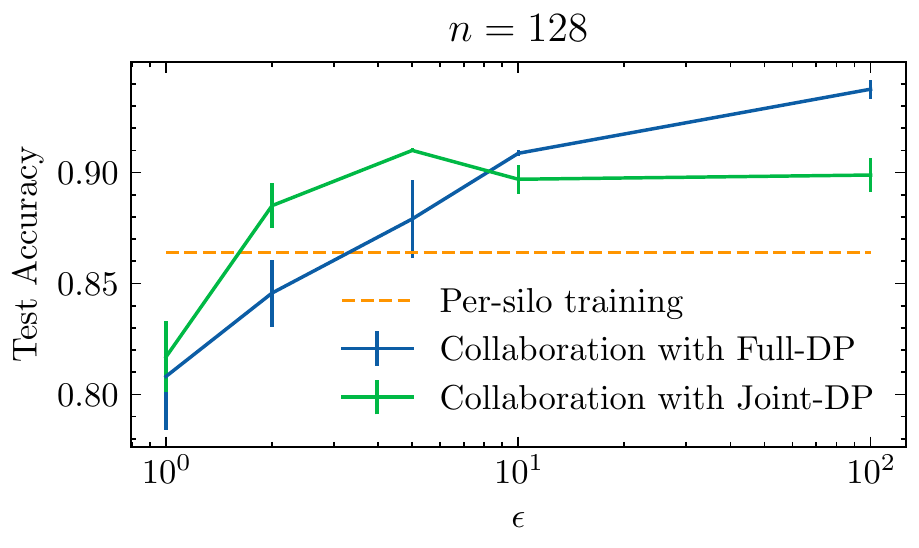}
    \includegraphics[width=0.32\linewidth]{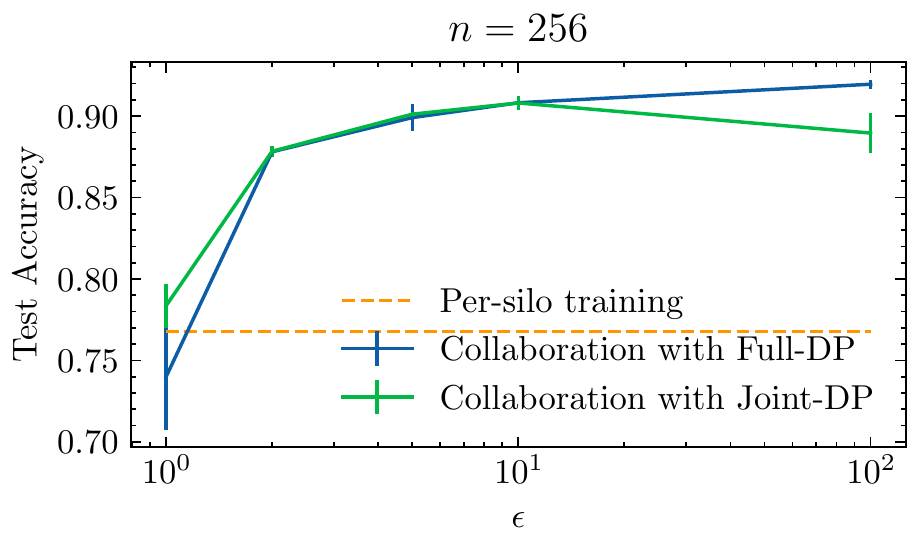}
    \includegraphics[width=0.32\linewidth]{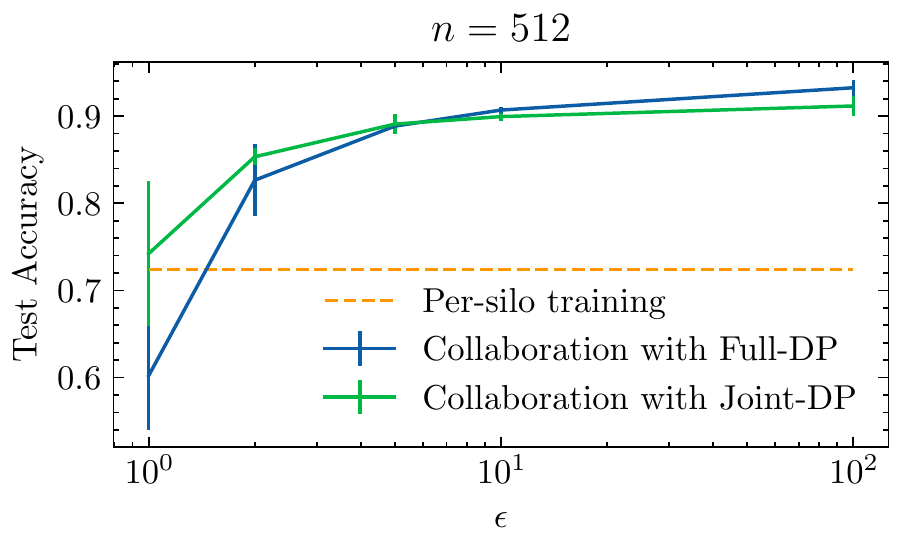}
    \vskip -0.1in
    \caption{Test accuracy of MNIST  for collaboration with full-DP and collaboration with joint-DP with different $\epsl$'s. We also report per-silo training results (the dashed lines) for reference.}
    \label{fig:MNIST_vary_n}
\end{figure*}

Note that although Collaboration without DP is included as a baseline in our experiment, it is not a standard practice in Federated Learning due to its lack of privacy protection. However, its performance helps us understand the utility upper bound for collaboration paradigms.

\subsubsection{Per-silo Training v.s. Collaboration without DP} Comparing the test accuracy of collaboration without DP to that of per-silo training allows us to understand the maximum performance potential of collaborative training. This can also provide insight into how factors such as $n$ (the number of owners) affect the performance. 

Our evaluation focuses on the setting where $m=r$ (i.e., each user owns 1 record), a crucial special case for the joint-DP formulation
\footnote{We provide results for $m>r$ in the Appendix.}.   We vary $n$, the number of owners in $\{4, 16, 64, 128, 256, 512\}$, and compare the test accuracy of collaboration without DP to that of per-silo training. As shown in Table~\ref{tab:main_table}, when $n \leq 64$, collaboration without DP achieves similar performance to per-silo training. This is because each owner already has enough data to train their own model, and therefore, collaboration does not provide a significant improvement in utility. However, as $n$ increases, collaboration without DP starts to perform much better than per-silo training, as it utilizes the combined data of multiple owners to improve the overall accuracy of the model.

\begin{table}[t]
\setlength{\tabcolsep}{12pt}
\begin{center}
\caption{Test accuracy of MNIST for collaboration with joint-DP under different configurations for individually trained parameters. This evaluation uses $n=128, \epsl=1.0$. The best result is highlighted in \textbf{boldface}.}
\label{tab:priv_config}
\begin{small}
\begin{sc}
\begin{tabular}{ccc|c|c}
\toprule
\multicolumn{3}{c|}{\bf Individually Trained?} & {\bf \# Individually} & {\bf Test Accuracy} \\
{Conv1} & {Conv2} & {FC1}   & {\bf  Trained Params}\\
\midrule
No & No & No &  $0$ & $0.7399 \pm 0.0324 $\\
Yes & No & No & $416$ & $0.7240 \pm 0.0130$  \\
No & Yes & No & $12,832$ & $0.7090  \pm 0.0014$ \\
Yes & Yes & No & $13,248$ & $0.1863 \pm 0.0329$\\
No & No & Yes &  $17,248$ & $\mathbf{0.7835} \pm 0.0136 $\\
Yes & No & Yes & $17,664$ & $0.6606 \pm 0.0259 $ \\
No & Yes & Yes & $30,080$ &   $0.7291 \pm 0.0320$ 
\\
Yes & Yes & Yes & $30,496$ & $0.1727 \pm 0.0419$ 
\\
\bottomrule
\end{tabular}
\end{sc}
\end{small}
\end{center}
\end{table}

\subsubsection{$(\epsl, \delta)$-DP models}

We then investigate the real-world setting where collaboration models are trained with differential privacy. We vary the privacy budget $\epsl \in \{1, 2, 5, 10, 100 \}$ and report the corresponding test accuracy for both joint-DP and full-DP models. We set $\delta = 1/N = 10^{-4}$ and the $\ell_2$ clipping norm to  15.

As shown in Table~\ref{tab:main_table} and Figure~\ref{fig:MNIST_vary_n}, when the privacy budget $\epsl \leq 10$, the performance of joint-DP models is consistently better than full-DP models. This improvement is due to the fact that a fraction of parameters in the joint-DP model is retained by each owner and therefore does not require the addition of DP noise. As the privacy budget $\epsl$ increases, full-DP models begin to achieve a higher accuracy because the added noise becomes smaller and no longer acts as a performance bottleneck; With noise becoming less of an issue, more parameters can be collaboratively trained among the owners in the model, which results in higher accuracy.

More interestingly, the improvement introduced by joint-DP in low-$\epsl$ regime is larger for larger $n$ (number of owners).
This observation is consistent with our theoretical results, as mentioned in Section~\ref{subsec:discussion}.
As the number of shared parameters is larger, hence $LD_{\U}/\sqrt{m}$, $LD_{\U}(\frac{1}{\sqrt{mn}}+\frac{\sqrt{(nk+\ell)\log(1/\delta)}}{\epsl mn})$ and $LD_{\U}(\frac{1}{\sqrt{mn}}+\frac{\sqrt{\ell \log(1/\delta)}}{\epsl mn})$ are the dominated terms when owners learn individually (i.e., per-silo training), collaboration with full-DP and collaboration under joint-DP respectively.
It is obviously $LD_{\U}/\sqrt{m}$ is independent of $n$, while the bounds of full-DP and joint-DP become smaller when $n$ gets larger, and the bound of joint-DP decreases faster.

\subsubsection{Choice of individually trained parameters}
\label{sec:exp_private_choice}

We also investigated various options for individually trained parameters. As previously mentioned, the model we use consists of two convolutional layers (Conv1 and Conv2) and two parallel fully connected layers (FC1 and FC2). We experiment with different combinations of individually trained or collaboratively trained for these layers. Notably, to differentiate from the per-silo training scenario where all parameters are individually trained, we consistently train FC2 collaboratively.

As shown in Table~\ref{tab:priv_config}, the test accuracy is poor when individually trained parameters are applied before collaboratively trained (i.e., shared) parameters. This is due to the fact that when the input to shared layers comes from the output of individually trained layers, the input distribution becomes highly non-IID, making it difficult to learn effective shared parameters.

\section{Conclusion}
In this paper, we initiate the study of stochastic convex optimization under joint differential privacy. Our theoretical results together with empirical studies provide very interesting insights into understanding how joint-DP models can provide better utility-privacy tradeoff compared to other baselines when learning across data owners with privacy constraints.
Unraveling the potential for improved bounds or illustrating lower bounds presents an engaging challenge.
We also believe that the general direction of studying machine learning under joint differential privacy is well-motivated and exciting.

\newpage
\bibliography{bib.bib}
\bibliographystyle{plain}

\newpage
\appendix
\section{Comparisons to Other Settings}
\label{sec:other}


\subsection{Owners Learn Individually}
Each owner minimizes their own $f_j(x_j,u)$ obviously satisfying the joint-DP guarantee in the billboard model, as the server does not output anything. The optimal excess population loss for each owner is 
\[
O \Big(L \sqrt{\frac{(D_{\X}^2 + D_{\U}^2)}{m}} \Big) \approx O \Big( L(D_{\X} + D_{\U})/\sqrt{m} \Big) .
\]
This is the optimal bound for non-private stochastic convex optimization.
The problem with this solution is that $D_{\U}$ can be much much larger than $D_{\X}$ when $\dim(u) \gg \dim(x)$, which is typically the case in practical applications. Optimistically, we should hope that denominator in the term involving $D_{\U}$  is  $\sqrt{nm}$, since all $n$ owners share the same parameter $u$, while the term involving $D_{\X}$ is still roughly $L D_{\X} / \sqrt{m}$.


\subsection{Collaboration Without DP}
\label{subsec:without_DP}
In this subsection, suppose we allow users to collaborate without worrying about privacy issues. The best possible excess population loss in this setting is always a lower bound for the best possible excess population loss in the joint-DP setting. 
For this non-private setting, we can use stochastic gradient descent as follows. 
Run for a total of $T$ steps, where $T$ will be chosen later; in each step $t \in [T]$, randomly sample an owner $j \in [n]$ and take one more fresh sample $z_j^{(t)} \in \mathcal{S}_j$ from her dataset to update parameters
\[
(x_j^{(t+1)}, u^{(t+1)}) = (x_j^{(t)}, u^{(t)}) - \eta_t \cdot \nabla h(x_j^{(t)}, u^{(t)}, z_j^{(t)}) . 
\]
Let $g_t = \nabla h(x_j^{(t)}, u^{(t)}, z_j^{(t)})$ be the gradient in step $t$.
Recall the definition of the population function of all users (Equation \eqref{eq:objective_function}).
We can concatenate $g_t$ with $(n-1)$ zero vectors $``\mathbf{0}''\in\mathbb{R}^k$ on the coordinates corresponding to owner $i$'s parameters $x_i$ for $i\ne j$, and get the concatenated vector, denoted by $\overline{g}_t$,  as an unbiased estimation of the subgradient for the excess population loss $\nabla f(x,u)$.
Then $\|\og_t\|_2^2 \leq L^2$ due to Lipschitzness, and the stochastic gradient is unbiased, i.e., $\E[\og_t]=\nabla f(x^{(t)},u^{(t)})$. 
The domain has size $\sqrt{n D_{\X}^2 + D_{\U}^2}$. 
Note that we can run the above algorithm for roughly $T \approx mn$ steps before running out of fresh samples from any of the owners. 
Let $\hat{x} = \frac{1}{T} \sum_{i=1}^T x^{(i)}$ and $\hat{u} = \frac{1}{T} \sum_{i=1}^T u^{(i)}$ be the average of the solutions.
The classic analysis of SGD then gives the following bound on the excess population loss
\begin{align*}
    f(\hat{x}, \hat{u}) \leq L \cdot \sqrt{\frac{(n D_{\X}^2 + D_{\U}^2)}{mn}} .
\end{align*}

This is an essential improvement to the bound when users learn individually, and matches our intuition that the term involving $D_{\U}$ should be divided by $\sqrt{mn}$. As there is no privacy guarantee in this setting, the above bound serves as a lower bound for the joint-DP setting that we primarily study in this paper. We analyze the excess population loss under joint-DP and full-DP in the following sections.


\subsection{Collaboration with Full DP}
As a comparison, we discuss the setting of full-DP: that is each owner will also share their personalized parameters $\{x_j\}_{j\in[m]}$ to other owners in the system so privacy should be taken into account in learning each owner's personalized parameters. 

\paragraph{Record Level.}
We consider the record level first.
We can add Gaussian noise $[B_t,b_t]$ to the (sub)gradient in $\A_{\rSGD}$ to obtain full DP, where $B_t\sim\mathcal{N}(0,\sigma^2I_{\X^n})$, where $\sigma=\frac{L^2T\log(1/\delta)}{\epsl^2 m^2n^2}$ and $T=\Theta(m^2n^2)$ are the same.

The privacy guarantee and error analysis are similar to the previous ones for joint DP. 

Analyze the generalization error. Denote $\widetilde{g}_t = \og_t + \mathcal{N}(0,\sigma^2 I_{nk+\ell})$, then we have
\begin{align*}
    \E[\|\widetilde{g}_t\|_2^2] =& L^2 + (nk+\ell) \sigma^2 = L^2 + \frac{(nk+\ell) L^2 \log(1/\delta)}{\epsl^2} \lesssim \frac{(nk+\ell) L^2 \log(1/\delta)}{\epsl^2}.
\end{align*}

So the convergence error would be 
\begin{align*}
\phi_{\opt}
& \lesssim \sqrt{\frac{(nk+\ell) L^2 \log(1/\delta)}{\epsl^2}} \cdot \sqrt{\frac{n D_\X^2 + D_\U^2}{m^2n^2}} = LR \cdot \frac{\sqrt{(nk+\ell)\log(1/\delta)}}{\epsl nm} . 
\end{align*}

Similarly, applying the uniform stability arguments can bound the generalization error by $O(RL/\sqrt{mn})$.
Hence the population loss is bounded as follows:
\begin{align*}
    \phi_{\risk}\lesssim RL \Big(\frac{1}{\sqrt{mn}}+\frac{\sqrt{(nk+\ell)\log(1/\delta)}}{\epsl mn} \Big).
\end{align*}

\paragraph{User Level.} As for the user-level DP, we can also apply group privacy and get a bound $O\big(RL(\frac{1}{\sqrt{mn}}+\frac{\sqrt{(nk+\ell)\log(m/\delta)}}{\epsl n})\big)$.

\section{Smooth functions}
This section compares user-level joint-DP SCO bounds for smooth function, whose proof basically based on \cite{LSA+21}.

We assume the loss function $h(\cdot, \cdot ,z)$ is $H$-smooth over $\X \times \U$.
The result is stated formally below:

\begin{theorem}[\cite{LSA+21}]
\label{thm:main_smooth}
For any $0<\delta<1/\poly(n)$ and $10\ge \epsilon>0$, setting $T=O(\frac{n^2}{\log(1/\delta)})$, Algorithm~\ref{alg:NSGD_SCO_jointDP_smooth} is an $(\epsilon,\delta)$ user-level joint-DP algorithm.
Moreover, if all data are drawn i.i.d. from the distributions, and the functions smoothness $H=O(\frac{G}{R}\frac{n\ell\log(\ell n)\sqrt{\log(1/\delta)\log m}}{\epsilon\sqrt{mr}})$, then it can achieve excess population loss
\[
\Tilde{O}\Big(RL \cdot \Big(\frac{1}{\sqrt{mn}} + \frac{\ell }{\epsl n \sqrt{mr}} \Big) \Big) ,
\]
where $R=\sqrt{n D_{\X}^2 + D_{\U}^2}$.
\end{theorem}

Compared to Theorem~\ref{thm:main_informal}, the second term in Theorem~\ref{thm:main_smooth} has a better dependence on $m/r$, that is when $m/r$ becomes larger the loss will become smaller, but a worse dependence on the dimension $\ell$.
The proof basically follows from \cite{LSA+21}.

\cite{LSA+21} proposed an user-level joint-DP algorithm for mean estimation when the data is uniformly concentrated, which is defined below:

\begin{definition}[$(\tau,\gamma)$-concentrated]
\label{def:uniformly_concentrated}
A random sample set $\{X_i\}_{i\in[n]}$ is $(\tau,\gamma)$-concentrated if there exists $x_0$ such that with probability at least $1-\gamma$,
\begin{align*}
    \max_{i\in[n]}\|X_i-x_0\|\le \tau.
\end{align*}
\end{definition}

For uniformly concentrated data, \cite{LSA+21} proposed a private algorithm to estimate the mean.

\begin{theorem}[\cite{LSA+21}]
\label{thm:mean_estimation}
    Consider sample set $\{X_i\}_{i\in[n]}$ which is $(\tau,\gamma)$-concentrated.
    There is an $(\epsilon,\delta)$-DP estimator $\ALG(\{X_i\}_{i\in[n]},\epsilon,\delta)$ such that $\ALG(\{X_i\}_{i\in[n]},\epsilon,\delta)\sim_{\beta}\ALG'(\{X_i\}_{i\in[n]},\epsilon,\delta)$ where 
    \begin{align*}
        &~\E\ALG'(\{X_i\}_{i\in[n]},\epsilon,\delta)=\frac{1}{n}\sum_{i\in[n]}X_i, \\
        &~\Var\Big(\ALG(\{X_i\}_{i\in[n]},\epsilon,\delta)\mid \{X_i\}_{i\in[n]}\Big)\lesssim\frac{\ell\tau^2\log(\ell n/\gamma)\log(1/\delta)}{n^2\epsilon^2},
    \end{align*}
    where $\beta=\min\{1,2\gamma+\frac{\ell^2R\sqrt{\log(dn/\gamma)}}{\tau}\exp(-n\epsilon/24\sqrt{\ell\log(1/\delta)})\}$.
\end{theorem}

Let $\calS_{j,k}$ denote the set of data of $k$-th user of $j$-th owner.
Recalling that the loss function is $L$-Lipschitz, we have the following concentration result:

\begin{lemma}[\cite{LSA+21}]
\label{lm:concentration_empirical_gradient}
Suppose $\{z_{ij}\}$ in $\calS_{j,k}$ are drawn i.i.d. from $\mathcal{P}_j$, then with probability greater than $1-\gamma$ that
\begin{align*}
    \sup_{x,u}\|\nabla f_{\mathcal{P}_j}(x,u)-\nabla f_{\calS_{j,k}}(x,u)\|=O\left( \frac{L\sqrt{r}}{\sqrt{m}}\big(\sqrt{\log(1/\gamma)}+\sqrt{\ell\log(RHm/\ell L)}\big)\right).
\end{align*}
\end{lemma}

This means the gradients of the empirical function of each user, $-\nabla f_{\calS_{j,k}}(x,u)$, can be $(\tau,\gamma)$-concentrated for some $\tau$ corresponding to the bound in the lemma above.

The critical difference of smooth function is that the final convergence rate of optimization can depend on the variance of the gradient estimations, rather than the second moment of the gradient estimations.
One can apply Theorem~\ref{thm:mean_estimation} to get private gradient estimations and apply SGD.

The algorithm is stated below:
\begin{algorithm}[h]
\caption{$\A_{\NSSGD}$: Noisy SGD for smooth functions under joint-DP}
\label{alg:NSGD_SCO_jointDP_smooth}
\begin{algorithmic}[1]
\STATE Dataset $\mathcal{S}$, $L$-Lipschitz loss $h(\cdot,\cdot)$, parameter $(\epsl, \delta)$, convex sets $(\X, \U)$, learning rate $\eta$  
\STATE Initial parameter $x_0 \in \mathcal{X}$ chosen arbitrarily
\FOR{$t = 0, \cdots, T-1$}
\STATE Sample $j_t \sim [n]$ uniformly at random 
\STATE For each user $k_t\in[r]$, compute $X_{k_t}=\frac{r}{m}\sum_{z_{i_t,j_t}\in \calS_{j_t,k_t}}\nabla h(x_{j_t}^t,u^t,z_{i_t,j_t})$
\STATE Get $g_t\leftarrow \ALG(\{X_{k_t}\}_{k_t\in [r]},\epsilon,\delta)$
\STATE $(x_{j_t}^{t+1}, u^{t+1}) \leftarrow 
\mathrm{update}((x_{j_t}^{t}, u^t),g_t)$ \label{ln:update_step}
\ENDFOR
\STATE \textbf{Return} $ (\overline{x}, \overline{u}) := \frac{1}{T} \sum_{t=0}^{T-1} (x^t, u^t)$
\end{algorithmic} 
\end{algorithm}

Let $\calS_{j,k}$ denote the $r$ data of user $k$ of owner $j$.
Let $f_{\calS_{j,k}}(x,u):=\frac{1}{r}\sum_{z_{ij}\in \mathcal{S}_{j,k}}h(x,u,z_{ij})$.

\begin{lemma}[\cite{bubeck2015convex}]
\label{lm:S-MD}
Let $F:X\to \R$ be a convex and $H$-smooth function, where $X$ is a convex set of diameter $R$.
Suppose for each step $x_t\in X$, there is an oracle that can output an unbiased stochastic gradient $\Tilde{g}(x_t)$ such that
\begin{align*}
    \E[\|\Tilde{g}(x_t)-\nabla F(x_t)\|^2]\le \sigma^2,
\end{align*}
then stochastic mirror descent with appropriate step size satisfies that
\begin{align*}
    \E[F(\frac{1}{T}\sum_{t=1}^Tx_t)]-\min_{x^*\in X}F(x^*)\le R\sigma\sqrt{2/T}+\frac{HR^2}{T}.
\end{align*}
\end{lemma}

Now we are ready to prove Theorem~\ref{thm:main_smooth}.

\begin{proof}
The privacy guarantee follows directly from the privacy guarantee of $\ALG$ (Theorem~\ref{thm:mean_estimation}) and privacy amplification by subsampling (\cite{}).

As for the utility guarantee, by Lemma~\ref{lm:concentration_empirical_gradient} and union bound, we know $\{X_{k_t}\}_{k_t\in[r]}$ is $(\tau,\gamma)$-concentrated for each $t\in[T]$, where $\tau=O\left( \frac{G\sqrt{r}}{\sqrt{m}}\big(\sqrt{\log(n\ell/\gamma)}+\sqrt{d\log(RHm/\ell G)}\big)\right)$.
Then by Theorem~\ref{thm:mean_estimation}, we know the variance of the stochastic gradient estimation is bounded by $O(\frac{\ell\tau^2\log(\ell n/\gamma)\log(1/\delta)}{r^2\epsilon^2})$.
Setting $\gamma=1/\poly(n)$ to be small enough, and
applying the convergence rate of stochastic mirror descent (Lemma~\ref{lm:S-MD}), we can get empirical loss
\begin{align*}
    O(LR\frac{\ell\log(1/\delta)\log(\ell n)\sqrt{\log m}}{n\epsilon\sqrt{mr}}+\frac{HR^2\sqrt{\log(1/\delta)}}{n^2}).
\end{align*}

We get the empirical loss by applying $H=O(\frac{L}{R}\frac{n\ell\log(\ell n)\sqrt{\log(1/\delta)\log m}}{\epsilon\sqrt{mr}})$ as in the precondition.
To get the population loss, as shown in \cite{LSA+21}, one can solve the phased Phased ERM and apply the localization technique.
\end{proof} 

\section{Experimental Details and More Results}
\label{sec:app_exp}

\paragraph{Federated Learning details.} In our experiments, we use FedAvg~\cite{medsa17}, and adjust $E$ (the local iterations) and $T$ (the global epochs) for different $n$, according to Table~\ref{tab:fl_config}.

\begin{table}[H]
\setlength{\tabcolsep}{6pt}
\begin{center}
\caption{The local iterations $E$ and the global epochs $T$ we use in the experiments.}
\label{tab:fl_config}
\begin{small}
\begin{sc}
\begin{tabular}{lcccc}
\toprule
\textbf{$n$} & \textbf{Local epoch $E$} &  \textbf{Global round $T$} \\
\midrule
$4$   & $20$  & $10$           \\
$16$   & $10$  & $10$           \\
$64$   & $10$  & $10$           \\
$128$   & $5$  & $20$            \\
$256$   & $5$  & $20$              \\
$512$   & $5$  & $20$             \\
\bottomrule
\end{tabular}
\end{sc}
\end{small}
\end{center}
\end{table}

\paragraph{Effect of $m/r$.} We investigate the impact of the ratio $m/r$ on the comparative performance between full-DP and join-DP. The number of owners ($n$) and the number of records per owner ($m$) are kept constant in our experiments. We vary  $r$, the number of users, to obtain different $m/r$ ratios, i.e., the number of records per user. 

We evaluate two scenarios: a small $n=4$ and a larger $n=256$. The experimental results, shown in Table \ref{tab:vary_m}, indicate that as the $m/r$ ratio increases, both full-DP and join-DP methods experience a decline in performance. However, joint-DP consistently outperforms full-DP. When $n=256$ and $\epsl > 1$, joint-DP could occasionally outperform per-silo training when $m/r$ > 1.
When $m$ and $n$ are fixed, increasing $r$ can get better performance.
This observed behavior aligns with our user-level theoretical result (\Cref{thm:main}).

\begin{table}[H]
\setlength{\tabcolsep}{2.5pt}
\begin{center}
\caption{Test accuracy of MNIST across 5 different runs for per-silo training ($\epsl=\infty$), collaboration without DP, with full-DP, and with joint-DP. We fix $n$ and $m$ and vary $m/r$, the number of records per user. DP results that are better than the per-silo training are highlighted using \underline{underline}.
}
\label{tab:vary_m}
\begin{scriptsize}
\begin{sc}
\begin{tabular}{c|c|cc|cccccc}
\toprule
 \multirow{2}{*}{ $m/r$ } &  \textbf{Per-silo} & {\textbf{Collaboration}} & {\textbf{Collaboration}} & \multicolumn{3}{c}{{\textbf{Collaboration w/ joint-DP}}}\\
&  \textbf{Training}& without DP & w/ full-DP ($\epsl=1.0$)  & $\epsl=1.0$ & $\epsl=3.0$ & $\epsl=8.0$ \\ 
\midrule
\multicolumn{7}{c}{$n=4$} \\
\midrule
 $1$ & \multirow{5}{*}{$0.9445 \pm 0.0069$}     & \multirow{5}{*}{${{0.9602}} \pm 0.0017$}  & $0.4444 \pm 0.0844$   & ${0.4658} \pm 0.0307$ & $0.5841 \pm 0.0656$ &  $ 0.8103 \pm 0.0177$ \\
 $2$ &  & & $0.2867 \pm 0.0387$ & $0.3297 \pm 0.0256$ & $0.5011 \pm 0.0616$  & $0.6576 \pm 0.0502$\\
 $5$ & & &  $0.1986  \pm 0.0479$  & $0.2064 \pm 0.0470$ & $0.2867 \pm 0.0315$ & $0.4802 \pm 0.0365$\\
 $10$ & & & $0.1526 \pm 0.0022$ & $0.1708 \pm 0.0262$   & $0.2644 \pm 0.0656$ & $0.3709 \pm 0.0581$\\
 $20$ & & & $0.1517 \pm 0.0438$ & $0.1691 \pm 0.0628$  & $0.1864 \pm 0.0318$ & $0.2741 \pm 0.0196 $\\
\midrule
\multicolumn{7}{c}{$n=256$} \\
\midrule
$1$ & \multirow{5}{*}{$0.7678 \pm 0.0015$}              & \multirow{5}{*}{${{0.9196}} \pm 0.0028$}  & $0.7399 \pm 0.0324$   &$\underline{0.7835} \pm 0.0136$ & $\underline{0.8124} \pm 0.0262$ & $\underline{0.8896} \pm 0.0124$ \\ 
 $2$ &  & & $0.2606 \pm 0.0259$ & $0.3930 \pm 0.0622$ & $0.6962 \pm 0.0375$ & $\underline{0.8160} \pm 0.0218$ \\
 $5$ & & & $0.1700 \pm 0.0154$  & $0.2517 \pm 0.0355$ & $0.3715 \pm 0.0642$ & $0.7020 \pm 0.0964$ \\
 $10$ & & & $0.1450 \pm 0.0227$  & $0.1603 \pm 0.0204$ & $0.2764 \pm 0.0620$ & $0.4703 \pm 0.1100$ \\
 $20$ & & & $0.1089 \pm 0.0264$ & $0.1297 \pm 0.0082$ & $0.2400 \pm 0.0329$ & $0.2941 \pm 0.0659$\\
  
\bottomrule
\end{tabular}
\end{sc}
\end{scriptsize}
\end{center}
\end{table}


\end{document}